\documentclass[conference]{IEEEtran}
\IEEEoverridecommandlockouts
\usepackage{cite}
\usepackage{amsmath,amssymb,amsfonts}
\usepackage{algorithmic}
\usepackage{graphicx}
\usepackage{xcolor}

\def\BibTeX{{\rm B\kern-.05em{\sc i\kern-.025em b}\kern-.08em
    T\kern-.1667em\lower.7ex\hbox{E}\kern-.125emX}}
    
\usepackage{amsthm}
\usepackage{array}
\graphicspath{{figures/}}

\newtheorem{theorem}{Theorem}    
    
\begin{document}

\title{Adaptive Model Learning of Neural Networks with UUB Stability for Robot Dynamic Estimation
\thanks{This work was published in 2019 International Joint Conference on Neural Networks (IJCNN). More information contact: pagand@sfu.ca}
}

\author{\IEEEauthorblockN{Pedram Agand}
\IEEEauthorblockA{\textit{Advanced Robotics and Automation Systems (ARAS),} \\
\textit{Department of Systems and Control, } \\
\textit{ K. N. Toosi University of Technology,}\\
Tehran, Iran.\\
aagand@email.kntu.ac.ir}
\and
\IEEEauthorblockN{Mahdi Aliyari Shoorehdeli}
\IEEEauthorblockA{\textit{Advanced Process  Automation and Control (APAC),} \\
\textit{ Department of Mechatronics Engineering,}\\
\textit{ K. N. Toosi University
of Technology,}\\
Tehran, Iran. \\
aliyari@kntu.ac.ir}
}

\maketitle

\begin{abstract}
Since batch algorithms suffer from lack of proficiency in confronting model mismatches and disturbances, this contribution proposes  an adaptive scheme based on continuous Lyapunov function for online robot  dynamic identification. This paper suggests stable updating rules to drive neural networks  inspiring from model reference adaptive paradigm. Network structure consists of three parallel self-driving neural networks which aim to estimate robot dynamic terms individually. Lyapunov candidate is selected to construct energy surface for a convex optimization framework. Learning rules are driven directly from Lyapunov functions to make the derivative negative. Finally, experimental results on 3-DOF Phantom Omni Haptic device demonstrate efficiency of the proposed method.
\end{abstract}


\section{Introduction}
Model-free identification in robotic applications has been extensively discussed in recent literature \cite{cook2017towards}. Since many control architectures in robotics require meticulous model of robot, exigency of efficient identification methodology is inevitable \cite{wu2010overview}. Even by ignoring some high-frequency dynamics and frictions, finding the dynamic of complex parallel robots still somehow difficult. Besides the incontinence along with finding the physical equation, due to inaccuracy of the model, robust approaches should be employed to guarantee the stability of the overall system that is not satisfactory in some delicate applications due to insufficient performance \cite{chen2016wnn}. To overcome this issue, data-driven identification is suggested for improvement of model quality.

 It seems more likely that the breakthrough will come through the use of other more flexible and amenable nonlinear system modeling tools such as the neural network in the form of multilayer perceptron (MLP) and radial Basis function (RBF), fuzzy, Local Linear Model (LLM), ARX, etc \cite{agand2016particle}. Among them, neural network proves to be a powerful yet simple tool for the nonlinear identification problems, that have been used extensively in the areas of filtering, prediction (e.g. \cite{jia2016dynamic}), classification and pattern recognition (e.g. \cite{dolph2017deep}), system modeling (e.g. \cite{suykens2012artificial}) and control (e.g.   \cite{piche2017gain}). Worthiness of  nonlinear system identification, especially in robotics systems is undeniable, since control systems encountered in practice possess the property of linearity only over a certain range of operation \cite{liu2012nonlinear}. Identification in robotics can be considered as two different points of view. In the first perspective, identification are accomplish for calibration of Kinematics (e.g. \cite{nguyen2015calibration}). In the other hand, identification is utilized as a way to render actual dynamic (e.g. \cite{moradi2013neural}) and control applications (e.g. \cite{heusel2017gans,he2016adaptive}). Based on this classification, online and offline identification methodology has shade the world of science. 
 
Stable learning rules were proposed for feedback linearization network in a class of single-input-single-output systems with continuous Lyapunov candidate. The idea of driving adaptive rules directly from continuous time Lyapunov function was presented by \cite{polycarpou1996stable}, where there was no need for the pre-assumption of network construction errors bounds, since the rules are a smooth function of states. In \cite{seng2002Lyapunov}, radial basis function is used as an adaptive filter by constructing discrete time Lyapunov functions.  A robust modification term is added to the updating rules to reinforce identifier against model mismatches and runtime disturbances. New adaptive backpropagation type algorithm is introduced by \cite{man2006new} to eliminate disturbances. It is worth-mentioning to say that the utilized Lyapunov candidate  only includes  networks error  ($V(k)=\sum\beta^ke^2(k)$). In \cite{lim2009Lyapunov}, multilayer neural network is utilized for classification of multi-input-multi-output system. By using Taylor expansion and discrete time Lyapunov candidate, the stability of the learning rules were proven.

In this paper, an adaptive learning rule is adopted for the network structure presented by \cite{agand2016transparent} while preserving UUB stability. By this end, nonlinear activation function in hidden layer are linearized using Taylor expansion. Not only  are stability and convergence of the networks errors targeted in this paper, but also the speed of error convergence is also controlled by adjusting the defined tuning parameters. The set of three independent networks are guaranteed to converge  by defining Weighted Augmentation Error (WAE). By proposing a Lyapunov surface comprising set of augmentation errors and parameter variation, a framework for converging global minimum is established. Updating rules are driven directly from Lyapunov function making its derivative negative. Since neural network  can never fully fit the desired dynamic due to  network construction error, modeling  mismatches and noises, a robust modification approach is utilized to avoid parameter drifts. 

The reminder of this paper is organized as follows. In Sec. II, the problem is declared using mathematical relations. In addition, the structure of neural network to solve it is presented. Sec. III is devoted to learning algorithm and weights adjustment. Some discussion about convergence and stability are outlined in Sec. IV. Experimental results on a 3-DOF serial manipulator, Phantom Omni Haptic device is done in Section V.  Finally, the paper is concluded in Sec. VI.
\section{Problem Statement}
By a class of Euler-Lagrange equations for robot dynamic we have
\begin{equation}
M(X)\ddot{X}+C(X,\dot{X})\dot{X}+G(X)=\tau=J^T\mathcal{F}
\label{eq:01}
\end{equation}
where, $M, C, G$ are the Inertia, the Coriolis and centrifugal  and the gravity. $\tau, \mathcal{F}$ denotes task-space and workspace forces, respectively. Jacobian of the robot is denoted by $J$.
The final aim is to obtain dynamics terms (namely $M,C,G$) individually by  a set of excitatory inputs including $X,\dot{X},\ddot{X}$ (motion variables, namely position, velocity, and acceleration) using a gray-box identification framework. Three parallel MLP-network is considered  for each term as following:
\begin{equation}
\hat{y}_n(W^h,W^o)=\sum_{j=1}^N\Big(W_{nj}^oF\big(\sum_{i=1}^{P_n}W_{ji}^hx_i+W^h_{j0}\big)\Big)+W^o_{n0}
\label{eq:02}
\end{equation}
where $W^h, W^o$ are the hidden and output layer weights respectively, $n=\{1, 2, \cdots, N\}$, $N$ is  number of robot degrees of freedom, $P_n$  is equal to number of hidden layers  in each network. Inputs for each network ($x_i$) is illustrated in Fig. \ref{fig:str}.
The term $\hat{y}_n$ can be presented in the form of $\hat{M}_n,\hat{C}_n,\hat{G}_n$. Each of these targets construct their own local error as follows:
\begin{equation}
y_n^j-\hat{y}_n^j=e_n^j~~\forall j \in \{M, C, G\} ,~~n\in\{1,2,...,N\}.
\label{eq:03}
\end{equation}
This error can not be calculated in each step, since no target output exist for each subnetwork. Therefore, the error is conveyed to secondary layer as follows:
\begin{equation}
\tau_n-\hat{M}_n\ddot{X}-\hat{C}_n\dot{X}-\hat{G}_n =e_{1n}~~\forall n\in\{1,2,...,N\}.
\label{eq:04}
\end{equation}
The relation presented in Eq. (\ref{eq:04}) is not sufficient to optimize whole networks. Moreover, admission of other criteria are taken into account to confine feasible set of parameters.  By other words, representation of the error is conveyed to the secondary layer which resolve inherent interconnectivity of the structure. Transparency is an intrinsic specific of this structure, since dynamics terms are included explicitly by separate parallel networks.  Relations presented in (\ref{eq:05}) and (\ref{eq:06}) are driven from the skew-symmetric property of $(\dot{M}-2C)$.
\begin{equation}
\dot{\hat{M}}_{nn}-2\hat{C}_{nn}=e_{2n} ~~\forall n\in\{1,2,...,N\}.
\label{eq:05}
\end{equation}
\begin{equation} \begin{split}
\dot{\hat{M}}_{ni}+\dot{\hat{M}}_{in}-&2(\hat{C}_{ni}+\hat{C}_{in})=e_{3m} \\ &\forall i\neq n, m\in\{1,2,...,(N^2-N)/2\}.
\label{eq:06}
\end{split}\end{equation}
Furthermore, according to mass property that states inertia matrix is not rank deficient and should satisfies $\underline{\lambda} I_{n\times n} \leq M(X) \leq \bar{\lambda} I_{n\times n}$ for  $0\leq \underline{\lambda} \leq \bar{\lambda}$, another error will be defined as follows:
\begin{equation}
|\hat{M}-\lambda~I_{n\times n}|=e_{4n}, ~ \lambda=eig(\hat{M})~e^{\lambda_0/t}~~\forall n\in\{1,2,...,N\}
\label{eq:07}
\end{equation}
where $|.|$ denotes matrix determinate. To drive the updating rules, a multi-objective optimization problem with prescribed equality conditions is considered as
\begin{equation}
Min \{\mathbf{J}\}=e_1^Te_1,~~ s.t ~e_2=e_3=e_4=0
\label{eq:091}
\end{equation}
where,
\begin{equation}
 e^t_j(k)=[e_{j1}^t(k), e_{j2}^t(k),...,e_{jn}^t(k)]^T, ~\forall j\in\{1,2,3,4\}
\label{eq:08}
\end{equation}
in which the superscript $(t)$ shows the order of data and $(k)$ depicts epoch's number.
The cost function in (\ref{eq:091}) can be optimized by adding Lagrangian weights as
\begin{equation}
\mathbf{H}=e_1^Te_1+\lambda_1e_2^Te_2+\lambda_2e_3^Te_3+ \lambda_3e_4^Te_4.
\label{eq:092}
\end{equation}
Therefore encapsulated error in a form of WAE is constructed as follows representing a target solution:
\begin{equation}
\varepsilon^t(k)=[e_1^T (k),\lambda_1e_2^T (k),\lambda_2e_3^T (k),\lambda_1e_4^T (k)]^T, \lambda_i>1.
\label{eq:09r}
\end{equation}
\section{Learning Algorithm}
In this section, an adaptive learning rule is driven with Lyapunov stability rules. 
The trajectory of a dynamic system with an equilibrium point at origin is said to be uniformly ultimated bounded (UUB stable) with respect to $S$ as any Lyapunov level surface of $V$ (system Lyapunov candidate), if the derivative of $V$ is strictly negative outside of $S$. Hence, all trajectory outside $S$ must be converged toward it. However, the asymptotic convergence of the trajectory can not be inferred.
\begin{theorem}
If the updating rules of the prescribed network in  (\ref{eq:02}) are considered as
\begin{equation}
\dot{W}_i^h =-\gamma_{i1} \varepsilon \xi_i; ~~ \dot{W}_i^o=-\gamma_{i2} \varepsilon \zeta_i ~~ \forall i\in\{M ,C, G\}
\label{eq:30}
\end{equation}
where $\varepsilon$ is obtained from encapsulated error in (\ref{eq:09r}) and,
\begin{equation}
\zeta_i=\frac{\partial \hat{\tau}}{\partial W_i^o} ; ~~\xi_i=W_i^o \frac{\partial \zeta_i}{\partial W_i^h} ; i \in\{M, C, G\}
\label{eq:25}
\end{equation}
then the system has UUB stability in confined subset
\begin{equation}
\mathcal{S}:\{x| \|x-\hat{x}\| \leq \gamma \nu_0 \}
\label{eq:251}
\end{equation}
where $ \gamma>1$ is a tuning parameter which controls the stability margin of the system and $\nu_0$ is the minimum inevitable error in a network that can be confined by
\begin{equation}
\nu_0 =\sup_{x\in B} |f_0-\nu| <\infty
\label{eq:2521}
\end{equation}
where $f_0$ is Taylor expansion error and $\nu$ denotes the network reconstruction error.
\end{theorem}
\begin{proof}
First the following dynamic mapping between the networks error in Eq. (\ref{eq:09r}) ($\varepsilon$) and the modified error in Lyapunov candidate (e) is considered:
\begin{equation}
\alpha e(t)+\frac{d e(t)}{dt}=\varepsilon(t),
\label{eq:21}
\end{equation}
where $\alpha>0$ is a designing parameter that controls speed of error convergence. The steady-state response of the modified error ($e$) is equal to the proportion of the WAE ($\varepsilon$). The larger value that  $\alpha$  takes, the response time and the steady-state error is reduced, while more chattering appear in the updating weights. Thereby, one should  compromise for choosing this parameter. Now consider the Lyapunov candidate as follows that include whole set of variables in system:
\begin{equation}
V=\frac{1}{2} \Big(\|e\|_2^2 +\Gamma^{-1} \|\tilde{\theta}\|_F^2\Big),
\label{eq:22}
\end{equation}
where $\|.\|_F$ is the Frobenius norm and
\begin{equation}
\theta:\Big\{W_M^o,W_M^h,W_C^o,W_C^h,W_G^o,W_G^h\Big\}, 
\label{eq:23}
\end{equation}
\begin{equation}
\tilde{\theta}=\theta-\theta^*, \theta^*=\min_{\theta\in B} \Big(\sup_{x\in X} \varepsilon(t) \Big).
\label{eq:231}
\end{equation}
The forward dynamic error is constructed as following:
\begin{equation}\begin{split}
\tau=W_M^{o*}f_M(W_M^{h*}X)\ddot{X}+W_C^{o*}f_C\Big(W_C^{h*} \begin{bmatrix}
X\\ \dot{X}
\end{bmatrix}
\Big)\dot{X}+\\W_G^{o*}f_G(W_G^{h*}X)+\nu
\label{eq:23}
\end{split}\end{equation}
where $\nu$ is the network reconstruction error which was mentioned previously in (\ref{eq:2521}). This error can be reduced by increasing number of neurons or changing network construction. The estimated torque is driven as
\begin{equation}\begin{split}
\hat{\tau}=W_M^{o}f_M(W_M^{h}X)\ddot{X}+W_C^{o}f_C\Big(W_C^{h} \begin{bmatrix}
X\\ \dot{X}
\end{bmatrix}
\Big)\dot{X}+\\W_G^{o}f_G(W_G^{h}X).
\label{eq:24}
\end{split}\end{equation}
Taylor expansion on Eq. (\ref{eq:24}) will results:
\begin{equation}
\hat{\tau}=\widetilde{W}_M^o \zeta_M+\widetilde{W}_M^h \xi_M +\widetilde{W}_C^o \zeta_C+\widetilde{W}_C^h \xi_C +\widetilde{W}_G^o \zeta_G+\widetilde{W}_G^h \xi_G +f_0,
\label{eq:241}
\end{equation}
where $\widetilde{W}_i^j$ are the difference between current and initial weights and
\begin{equation}
 f_0=f_{0M}+f_{0C}+f_{0G}
 \label{eq:251}
\end{equation}
\begin{equation}\begin{split}
f_{0M}=W_M^of_M(W_M^hX)\ddot{X}-W_M^{o*}f_M(W_M^{h*}X)\ddot{X}-\\\widetilde{W}_M^o\zeta_M-
\widetilde{W}_M^h\xi_M.
\label{eq:262}
\end{split} \end{equation}
Eq. (\ref{eq:262}) holds for $C, G$ similarly. Derivative of Lyapunov function in Eq. (\ref{eq:22}) will be
\begin{equation}
\dot{V}=e^T\dot{e}+\Gamma^{-1} tr\{\dot{\theta} \tilde{\theta}^T\}
\label{eq:27}
\end{equation}
by using Eq. (\ref{eq:21})
\begin{equation}
\dot{V}=- \alpha \|e\|_2^2+e^T \varepsilon +\Gamma^{-1} tr\{\dot{\theta} \tilde{\theta}^T\}
\label{eq:272}
\end{equation}
\begin{equation} \begin{split}
\dot{V}=e^T\Big(\sum_{i=1}^3(\widetilde{W}_i^h\xi_i+\widetilde{W}_i^o\zeta_i) +f_0-\nu\Big)+&\\\sum_{i=1}^3 \Big(\frac{1}{\gamma_{i1}} tr\{\dot{W}_i^h \tilde{W}_i^h\}+\frac{1}{\gamma_{i2}}tr\{\dot{W}_i^o \tilde{W}_i^o\}\Big)& - \alpha \|e\|_2^2
\label{eq:28}
\end{split}\end{equation}
where $i=\{1,2,3\}$ assign to $\{M,C,G\}$ respectively. Moreover we have
\begin{equation}\begin{split}
\dot{V}=e^T(f_0-\nu)+\sum_{i=1}^3 \Big( \frac{1}{\gamma_{i1}} tr\{\dot{W}_i^h \tilde{W}_i^h+\gamma_{i1} \varepsilon \tilde{W}_i^h \xi_i\}\Big)+\\\sum_{i=1}^3 \Big( \frac{1}{\gamma_{i2} }tr\{\dot{W}_i^o \tilde{W}_i^o+\gamma_{i2} \varepsilon \tilde{W}_i^o \zeta_i\}\Big) -\alpha \|e\|_2^2
\label{eq:29}
 \end{split}\end{equation}
By using updating rules in (\ref{eq:30})  and (\ref{eq:2521}) we have
\begin{equation}
\dot{V}\leq e^T\nu_0-\alpha \|e\|_2^2
\label{eq:291}
\end{equation}
To guarantee stability in the presence of uncertainty and modeling error, robust control approaches are considered as margin for stability.  By this end, the prescribed error must not enter the ball of a definite radius corisponding to inevitable error in (\ref{eq:2521}). Consequently, due to uncertainty presented in system, the goal is not to vanish error, but is to make it greater than a value. This consideration is formulated as follows:
\begin{equation} 
|e(t)|\geq \gamma \frac{v_0}{\alpha}, \gamma >1
\label{eq:31}
\end{equation}
Implementation of this condition has been proposed widely in the literature known as $\sigma$-modification \cite{hovakimyan2002adaptive}, projection method, $\delta$-rule  dead zone  and other robust approaches \cite{ioannou2012robust}. These methods change the direction of movement in the case that leads to escape of parameter from prescribed subsets ($B$). These methods are utilized to confront with the bursting phenomenon in the adaptive systems by adding robust modification term.
  Derivative of Lyapunov function will be simplified as:
\begin{equation}
\dot{V}\leq -\alpha \|e\|_2^2+\frac{\alpha}{\gamma} e^T e.
\label{eq:32}
\end{equation}
Using Eq. (\ref{eq:31}) in (\ref{eq:32}) will results in
\begin{equation}
\dot{V}\leq -\alpha (1-\frac{1}{\gamma})\|e\|_2^2 <0.
\label{eq:33}
\end{equation}
Last relation proves the convergence of the algorithm to a limited subset, by the other word, UUB stability guarantees the boundedness on the weights and network error. Meanwhile, all the variables belong to the $\mathcal{L}_\infty$ space. 
\end{proof}

\section{Discussion}
 Since the regressors are reproduced from  output in a simulation-error-based identification which are in  ware of divergence for long time running.  According to Eqs. (\ref{eq:04}) to (\ref{eq:07})  four linear independent set of errors are defined. The proof of convergence for each subnetwork to their local values is investigated in our previous work in \cite{agand2017Adaptive}.

In Fig. \ref{fig:into}, a representation of the proposed algorithm with the robust modification is shown for identification of pole place in a first order dynamic system. In this figure, the continuous black line depict the network error in (\ref{eq:09r}), while the red line illustrate the trend of modified error in (\ref{eq:21}). On the other hand, the dashes line is a wise proportion of the network error, in a way to maintain the modified error in a valid zone.  Region I is the impermissible area, since the identifier is over-parametrized. Region II, is the stability margin of the identifier, which is controlled by parameter $\gamma$. Region III, is the operation area and declare the UUB subset. Region VI, is the low performance and unstable area. It is concluded that  increasing parameter $\alpha$ will enhance the speed convergence. However, the system stability is jeopardized, if $\gamma$ is not well chosen. According to the discussion, to restrict the error in region III, one should chose the tuning parameter satisfying $1<\gamma\leq\alpha$.

\begin{figure}[t!]
  \centering
  \begin{minipage}[b]{0.4\textwidth}
\includegraphics[width=\linewidth,trim={0cm  0cm 8cm 1cm },clip]{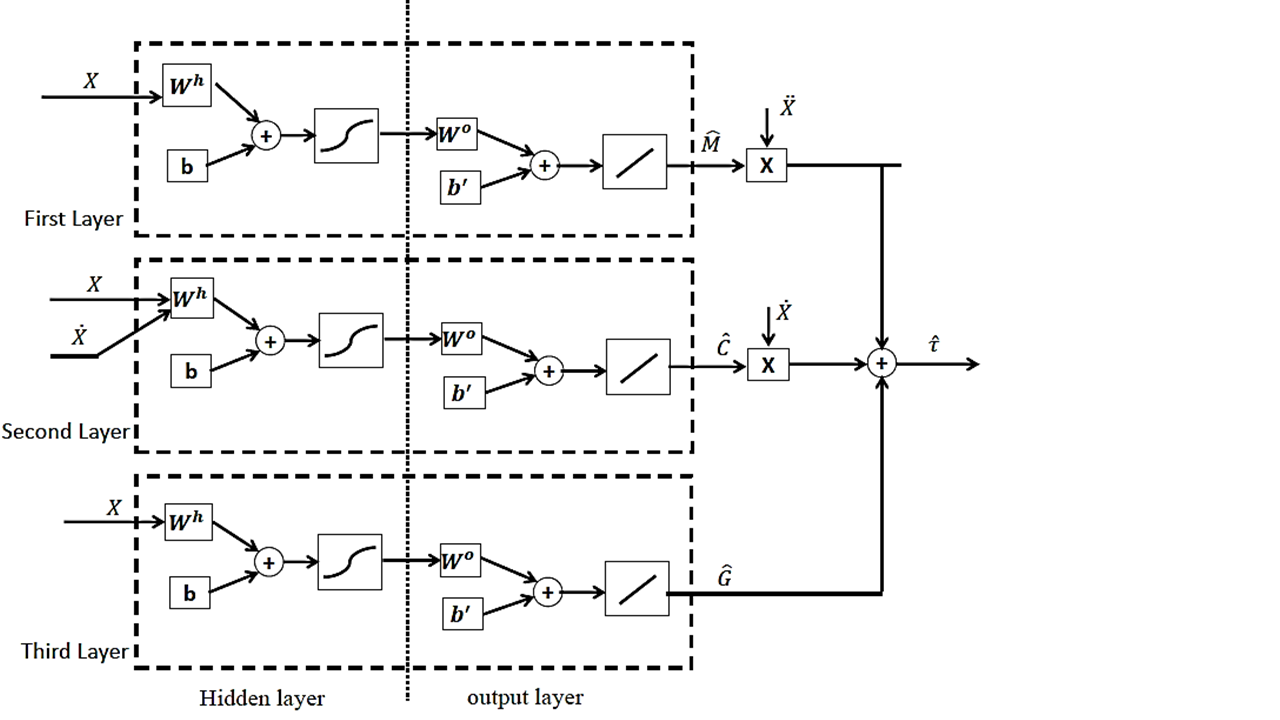}
  \caption{Proposed neural network structure.}
  \label{fig:str}
  \end{minipage}
  \hfill
  \begin{minipage}[b]{0.35\textwidth}
 \includegraphics[width=\linewidth,trim={6cm  1cm 5cm 2cm },clip]{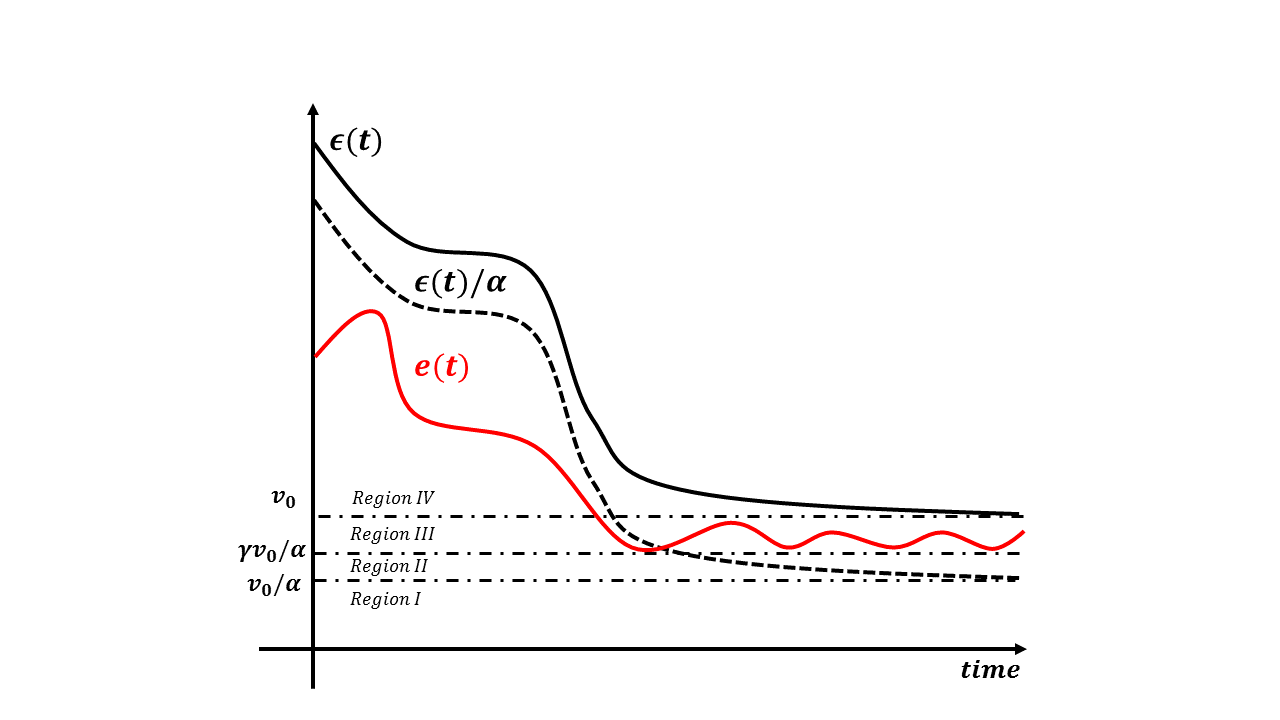}
  \caption{Effect of tuning parameters, $\alpha, \gamma$.}
  \label{fig:into}
  \end{minipage}
\end{figure}

\section{Experiments}
\subsection{Setup}
For the experimental, a 3-DOF Phantom Omni Haptic device, as serial manipulator designed for robotic application is used for testbed. The setup is shown in Fig. \ref{fig:phantom}. Since all regressors of network are not available (namely velocity and acceleration of joint variables), Kalman filter is used to smooth and predict input variables. The identification process is done in the close loop using a simple proportional controller not to deteriorate special frequency information. Data are extracted realtime in external mode Matlab S-function and 1 msec sampling rate. IEEE-1394a compliant six pin to six pin FireWire cable is used to connect the Phantom Omni device to the PC. Forward dynamic identification of the proposed network is shown in Fig. \ref{fig:tr1}. As it can be seen in this figure, the measured and estimated torque differs in temporarily response. The reason for that is two folded. Firstly, the neural network is updating it self and little by little is identifying the system. In addition, the modified error in the backend is performed as a  stable dynamic. After a period that the modified error (e) allocated in the third region (operation area), the network error is reached its minimum value and the tracking error mitigates but not vanished. For the sake of illustration, the effect of parameter $\alpha$ is checked in the temporary response of convergence. The impact of $\alpha$ on tracking error and weights adjustment are shown in Fig. \ref{fig:alpha1}, \ref{fig:weight}, respectively.  It is worth-mentioning that increasing the value of $\alpha$ will speedup the convergence of error. Moreover, the steady-estate error is reduced. However, as Fig. \ref{fig:weight} illuminate, the profile of weights adjustment is more outspread if the value of $\alpha$ increases.

The tuning parameters has direct impact on both error convergence and speed of the convergence. Consider the case, where $\alpha$ tends to infinity, where the derivative term in Eq. (\ref{eq:21}) is simply neglected, and apparently both speed of convergence and the amount of error are decimated. The proposed structure become a conventional back propagation strategy with gradient decent method which can be considered as a baseline method for comparison. However, in that case, the chartering effect arouses and deteriorate the overall performance of the system. In addition, due to insufficiency of model validity in the gray box identifier, beside the complexity of the convex optimization problem, stability and convergence of the NN jeopardizes grossly. Hence, to show superiority of the proposed strategy, the results are compared with a non-adaptive NN structure in the next section.

\begin{figure}[t!]
  \centering
\includegraphics[width=\linewidth,trim={0cm  0cm 0cm 0cm },clip]{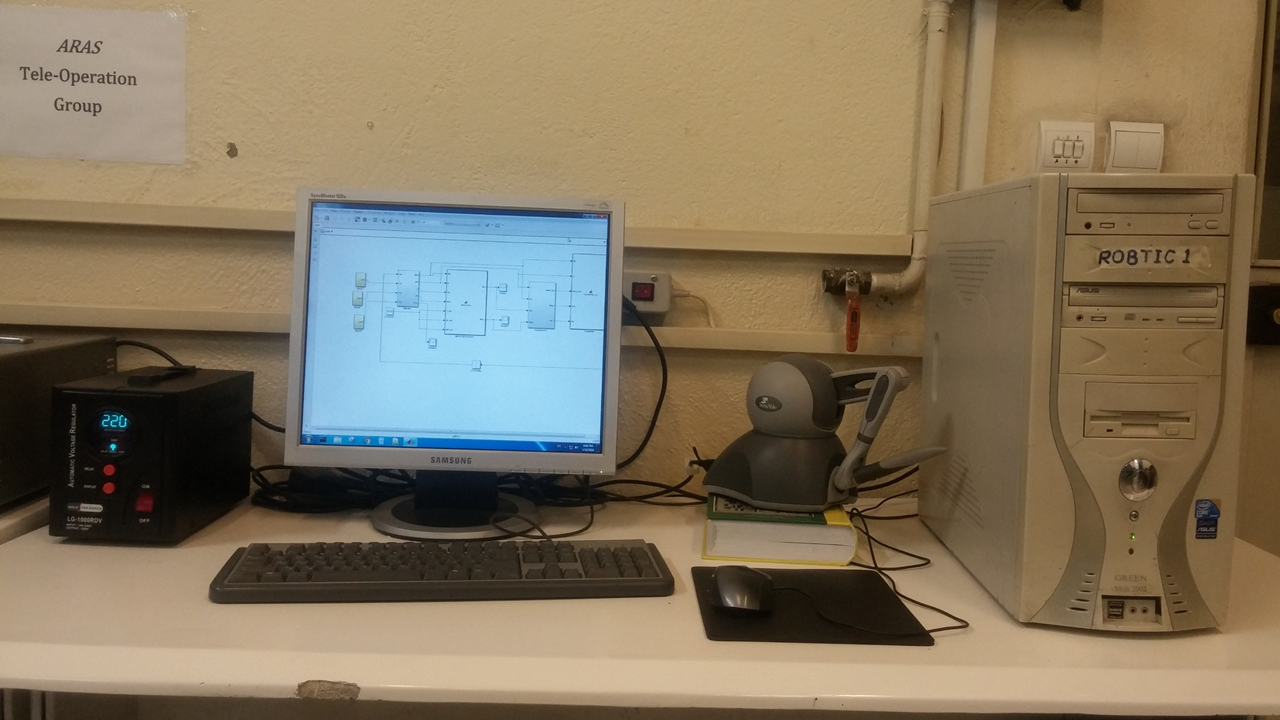}
  \caption{Phantom Omni, Sensable: Haptic Device.}
  \label{fig:phantom}
\end{figure}

\begin{figure}[t!]
  \centering
\includegraphics[width=\linewidth,trim={1cm  0.5cm 1.5cm 0.5cm },clip]{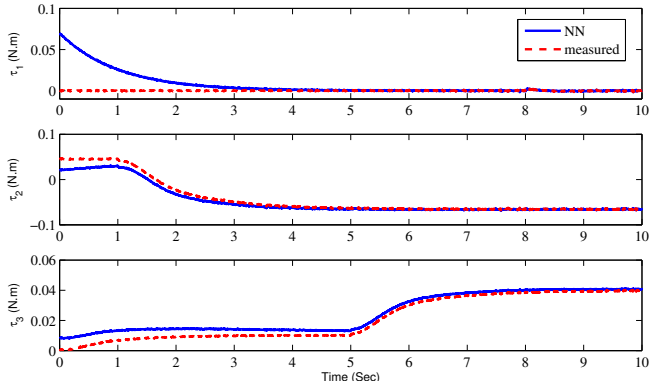}
  \caption{3-DOF forward dynamic tracking: measured and NN estimation.}
  \label{fig:tr1}
\end{figure}

\begin{figure}[t!]
  \centering
  \begin{minipage}[b]{0.48\textwidth}
 \includegraphics[width=\linewidth,trim={1cm  0cm 1.8cm 0.3cm },clip]{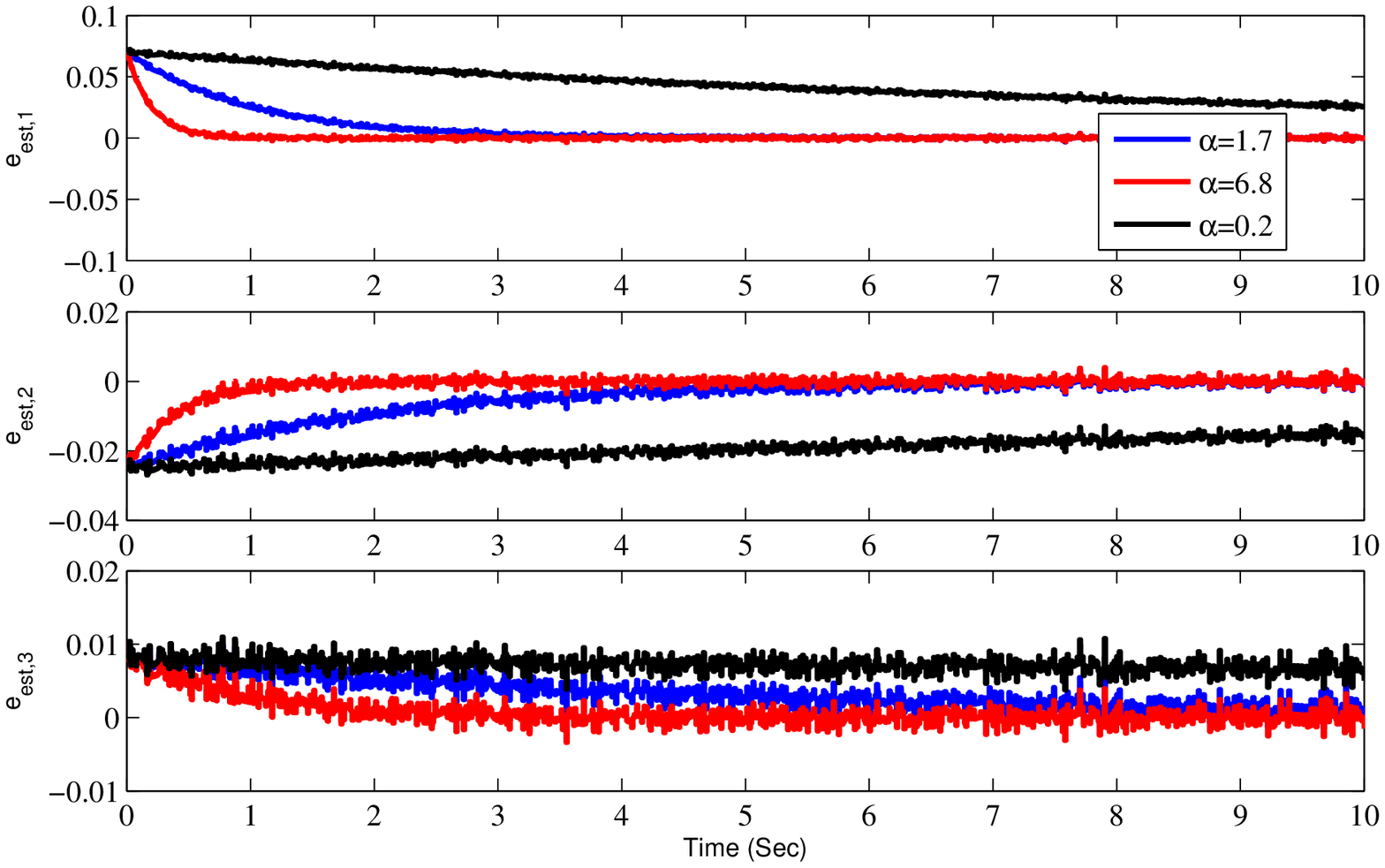}
  \caption{Investigating the effect of $\alpha$ in tracking error.}
    \label{fig:alpha1}
  \end{minipage}
  \hfill
  \begin{minipage}[b]{0.48\textwidth}
 \includegraphics[width=\linewidth,trim={0.5cm  0cm 1cm 0cm },clip]{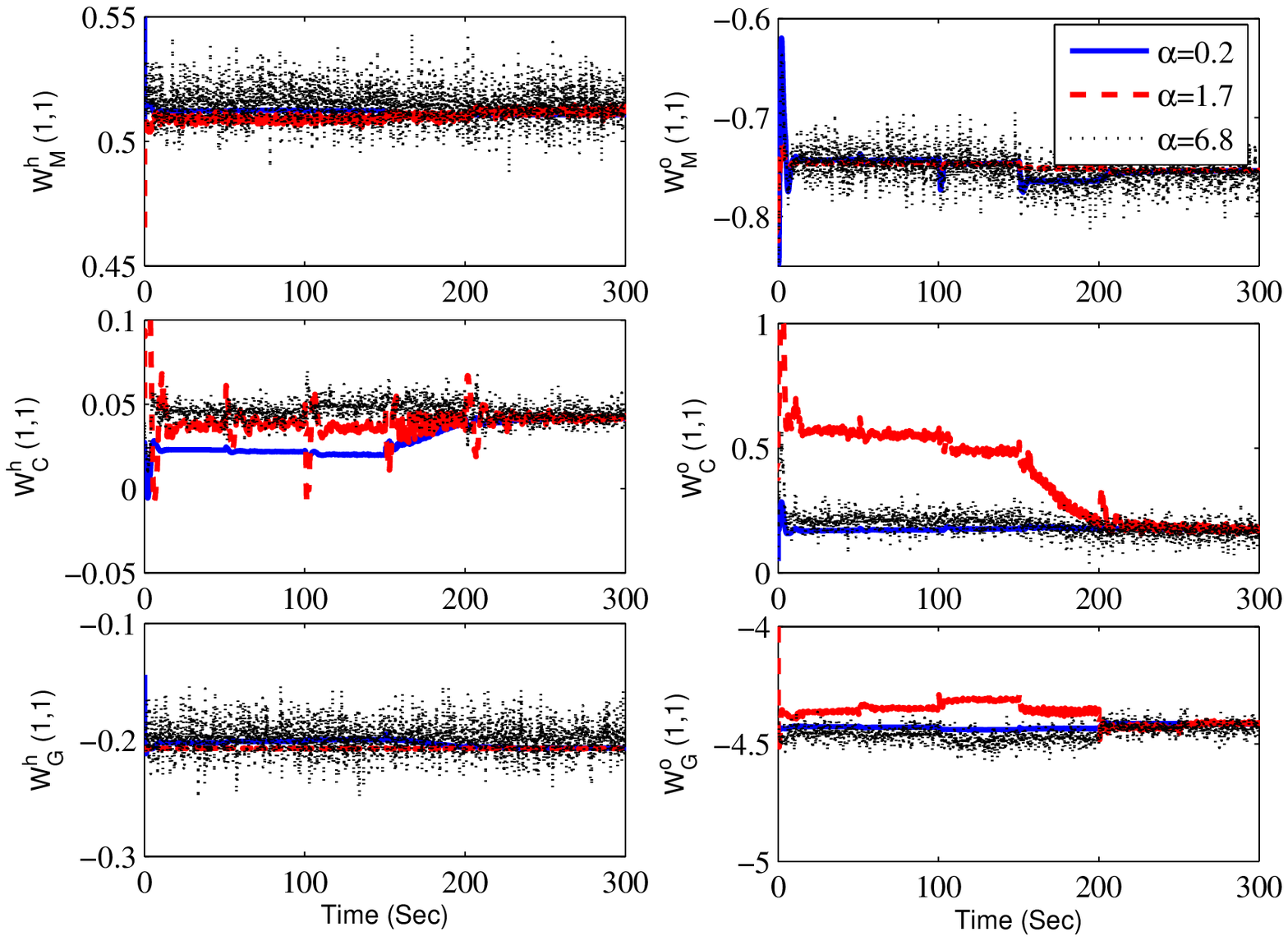}
\caption{Investigating the effect of $\alpha$ in weight adaptation.}
  \label{fig:weight}
  \end{minipage}
\end{figure}

\subsection{Validation}
For verification, an Inverse dynamic control (IDC) structure using the proposed neural network is utilized (NNIDC). According to Fig. \ref{fig:cont}, neural network tries to identify the dynamic terms individually online. In order to ameliorate undesirable response of adaptive structure, the weights are pre-tuned using batch data. The prescribe structure employ an indirect adaptive scheme to cancel nonlinearity presented in robot dynamic thanks to the feedback linearization. Then the proportional differential controller is used to track the reference signal. For the sake of comparison, two scenarios are suggest with/without NNIDC structure and the same PD controller.  The tracking of joint space motion variables in two scenarios are illustrated in Figs. \ref{fig:ver1}, \ref{fig:ver2} and \ref{fig:ver3} for each degrees of freedom. The priority of tracking performance in $q_2, q_3$ outweighs $q_1$ due to the compensation of affine terms of gravity. To have better insight about the interior of the prescribed structure, control effort of PD and NNIDC controller are depicted in Fig. \ref{fig:upi} and \ref{fig:uidc}, respectively. The effort signal in uncompensated structure is more chattering and larger in amplitude.

\begin{figure}[t!]
 \centering
 \includegraphics[width=\linewidth,trim={0.5cm  5cm 0cm 0.5cm },clip]{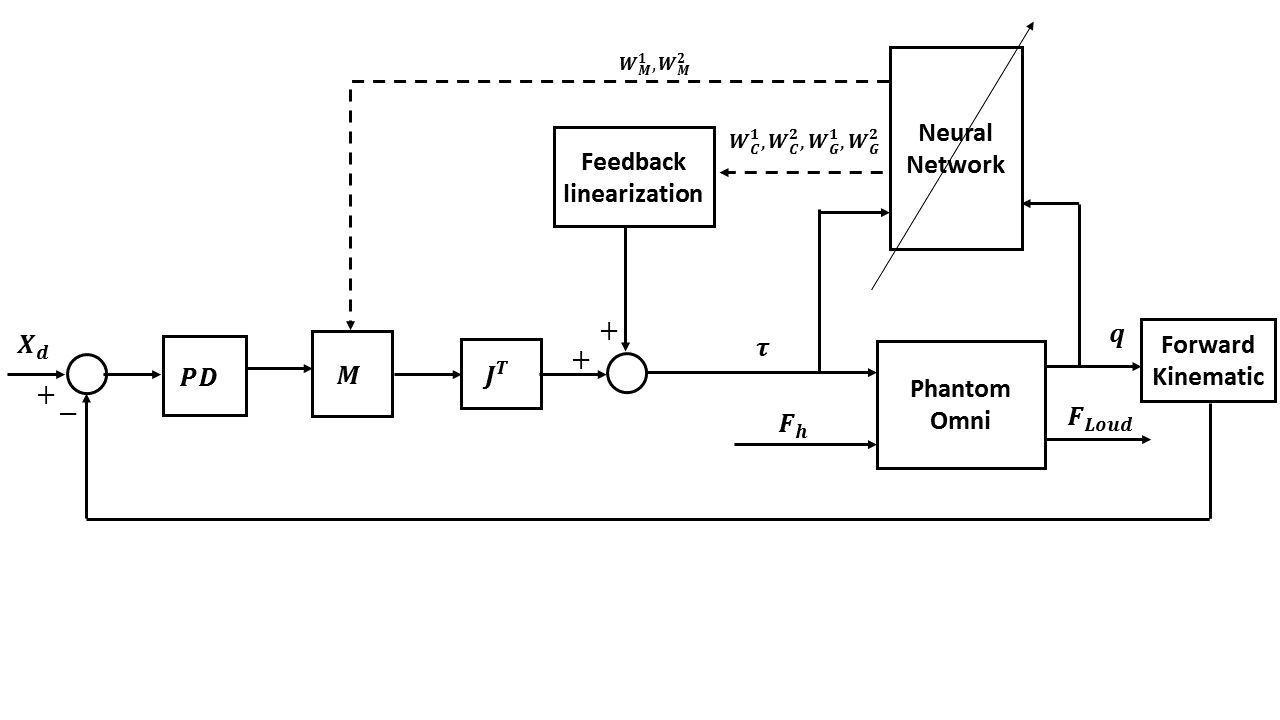}
  \caption{Indirect NN-based adaptive scheme for  robot control.}
  \label{fig:cont}
  \end{figure}
  
  \begin{figure}[t!]
  \centering
\includegraphics[width=\linewidth,trim={1cm  0.5cm 1.5cm 0.5cm },clip]{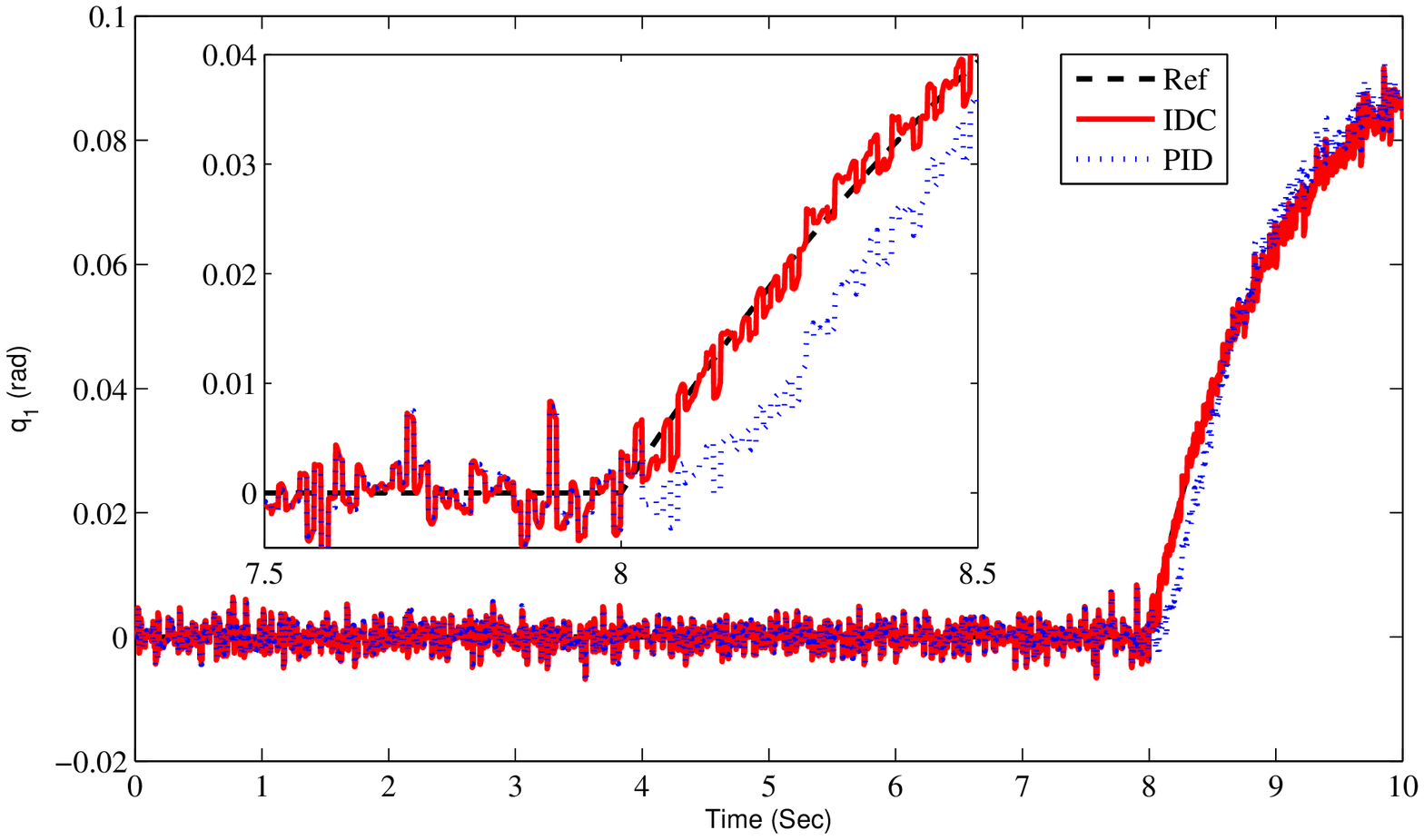}
  \caption{Comparison simple PD and NN IDC ($q_1$).}
  \label{fig:ver1}
  \end{figure}

\begin{figure}[t!]
  \centering
  \begin{minipage}[b]{0.48\textwidth}
\includegraphics[width=\linewidth,trim={0.5cm  0cm 1.5cm 0cm },clip]{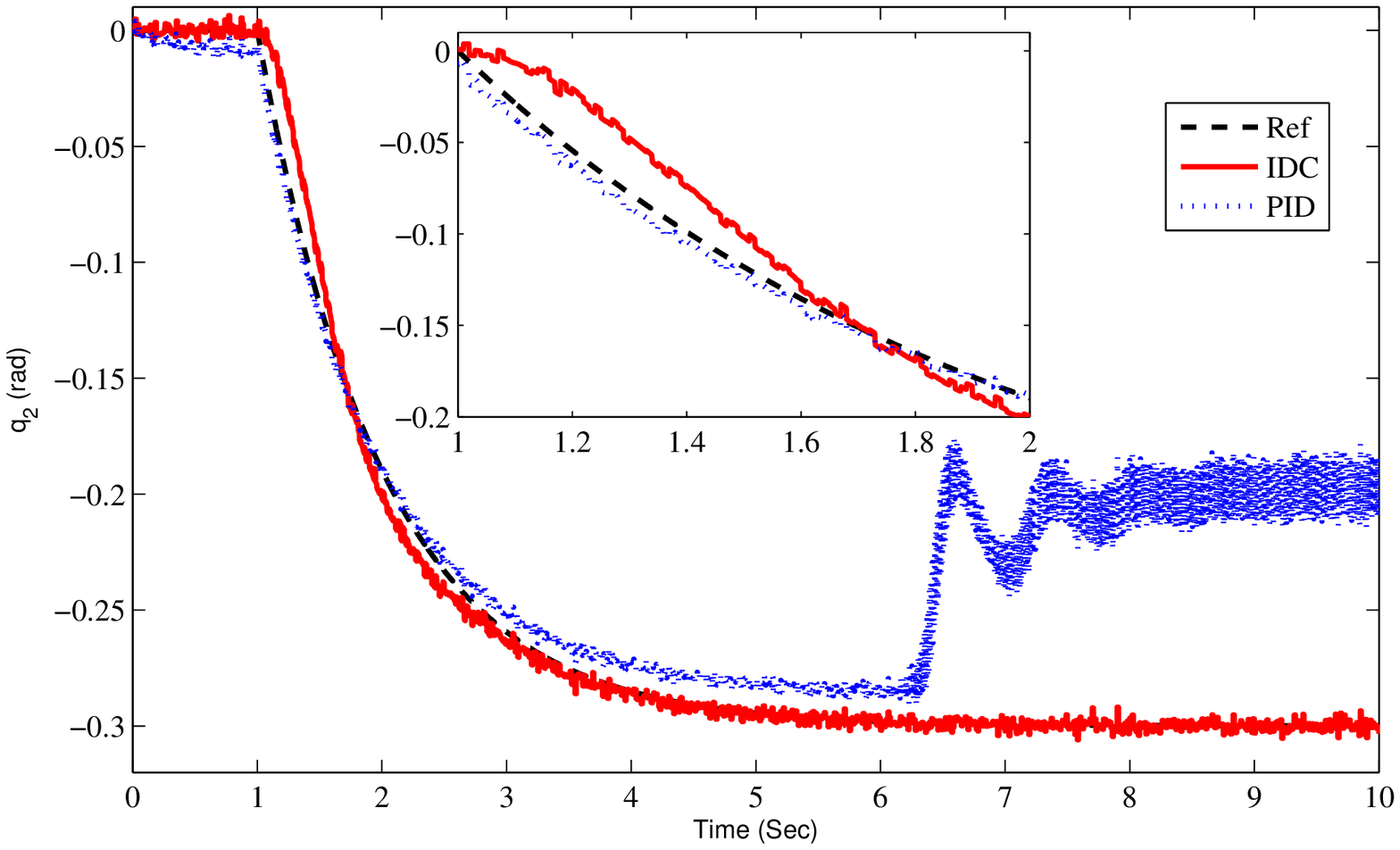}
  \caption{Comparison simple PD and NN IDC ($q_2$). }
  \label{fig:ver2}
  \end{minipage}
  \hfill
  \begin{minipage}[b]{0.48\textwidth}
 \includegraphics[width=\linewidth,trim={0.5cm  0cm 1.5cm 0cm },clip]{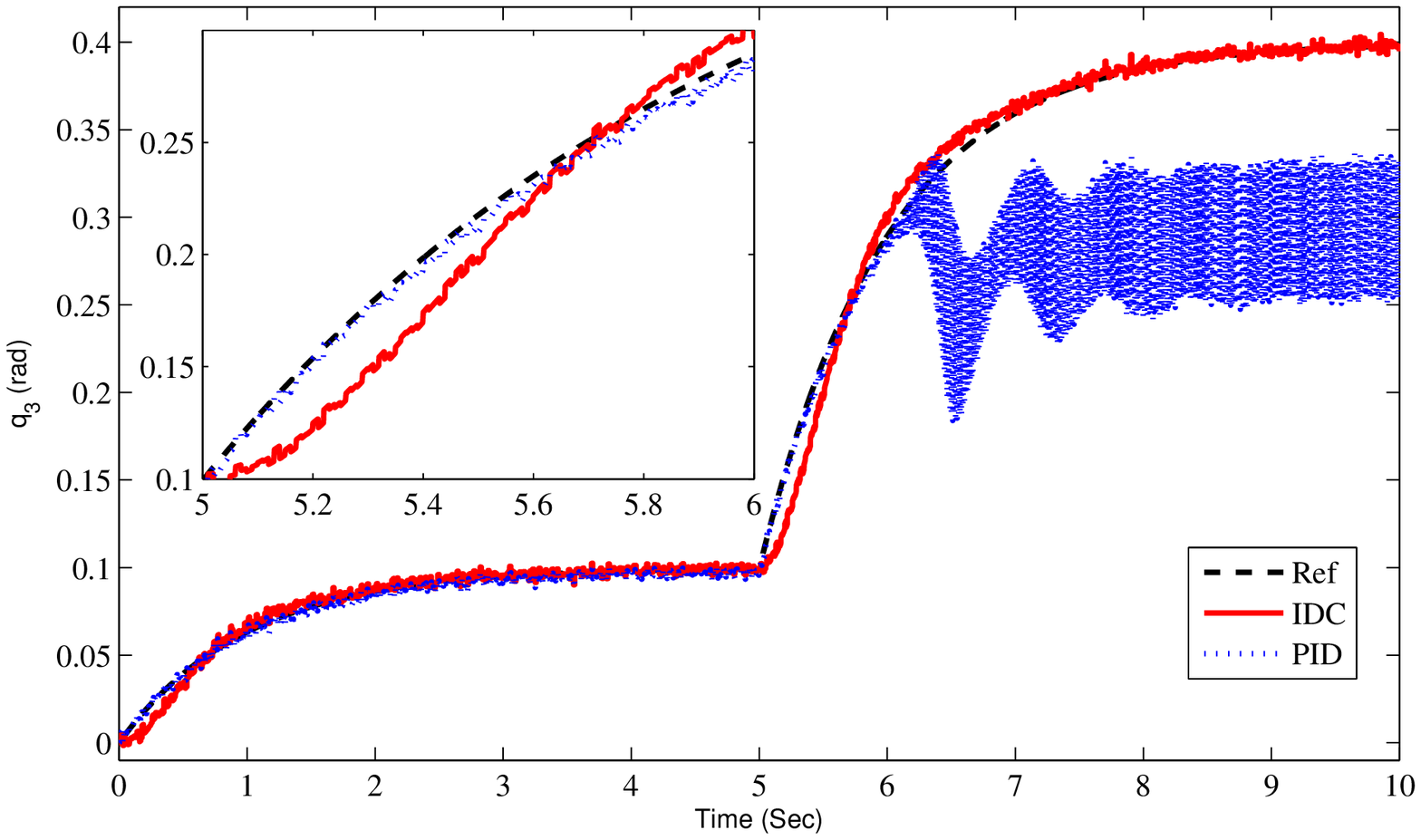}
  \caption{Comparison simple PD and NN IDC ($q_3$).}
  \label{fig:ver3}
  \end{minipage}
\end{figure}

\begin{figure}[t!]
  \centering
  \begin{minipage}[b]{0.48\textwidth}
 \includegraphics[width=\linewidth,trim={0.5cm  0cm 1cm 0cm },clip]{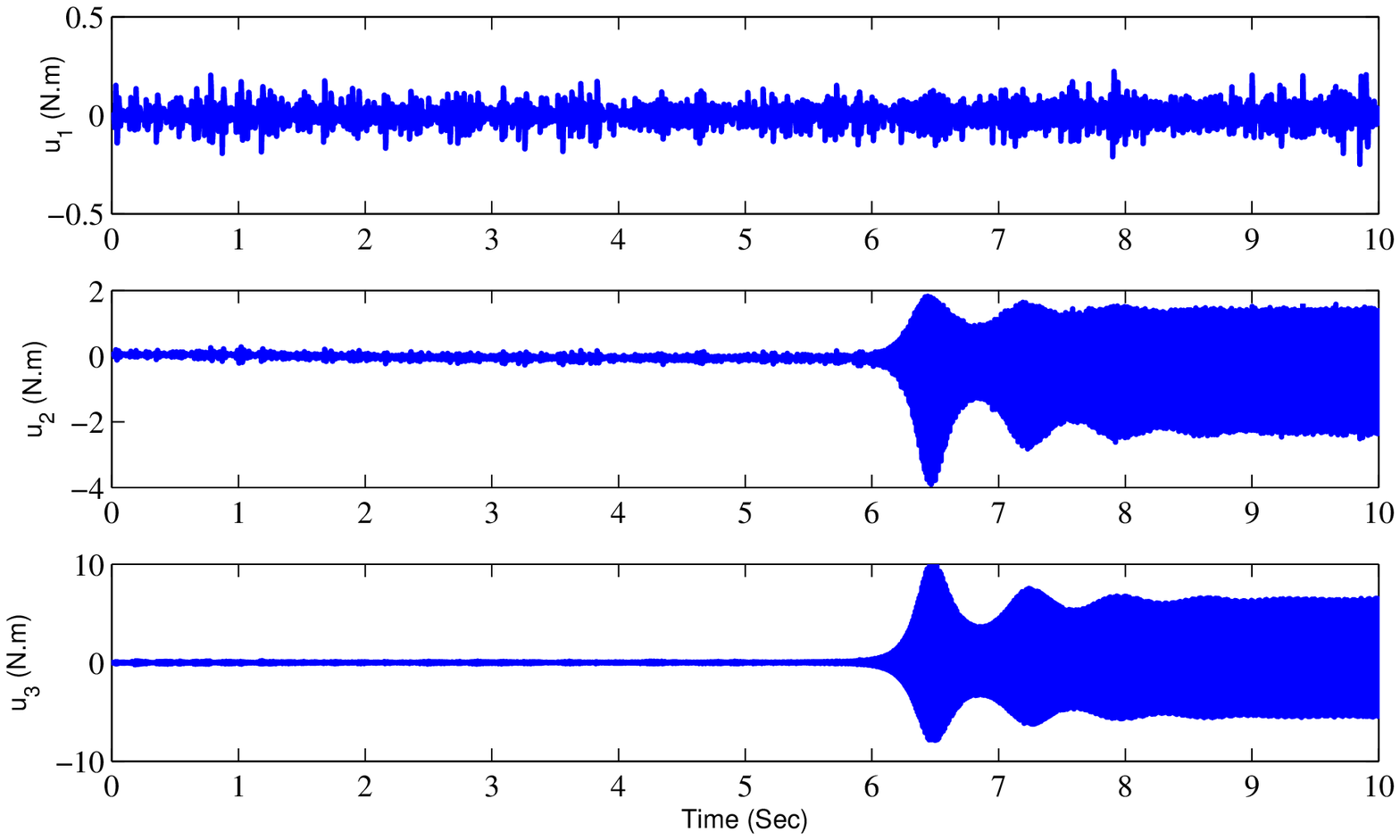}
  \caption{Control effort of PD controller.}
  \label{fig:upi}
  \end{minipage}
  \hfill
  \begin{minipage}[b]{0.48\textwidth}
 \includegraphics[width=\linewidth,trim={0.5cm  0.5cm 1cm 0cm },clip]{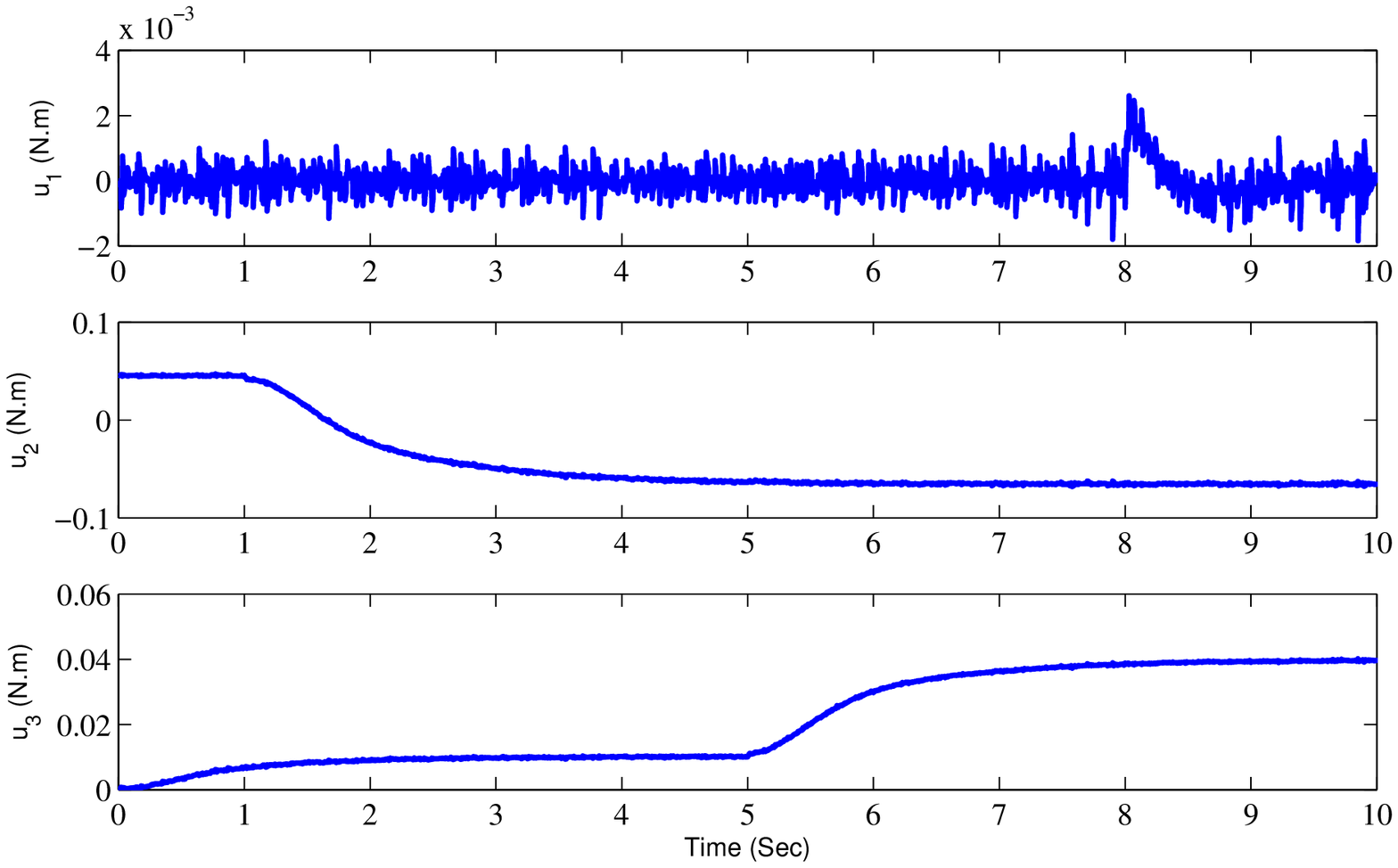}
  \caption{Control effort of NN-Based IDC controller.}
  \label{fig:uidc}
  \end{minipage}
\end{figure}

\section{Conclusion}
In this article, a learning algorithm is driven to identify robot dynamic terms individually with UUB stability. The proposed adaptive structure has the capability to be used in a great class of identification process.  By proposing two tuning parameters, one should compromise between  rate of convergence and stability.  Experimental setup on 3-DOF  Phantom Omni reveal applicability of the proposed structure and learning algorithm.

\addtolength{\textheight}{-3cm}  
\bibliographystyle{IEEEtran}
\bibliography{mybib}

\end{document}